\documentclass[11pt, a4paper, copyright, gdm]{google}

\usepackage[authoryear, sort&compress, round]{natbib}

\title{A Diffusion Analysis of Policy Gradient for Stochastic Bandits}
\renewcommand{\today}{2026-03-10}

\author[1]{Tor Lattimore}

\affil[1]{\thepa{}{}}

\uselogo{} 
\correspondingauthor{lattimore@google.com}

\usepackage{tikz}
\usetikzlibrary{positioning}

\usepackage{bbm}
\usepackage[plain]{algorithm}
\usepackage{thmtools} 
\usepackage[capitalize]{cleveref}
\usepackage{mathtools}
\usepackage{listings}
\usepackage{mdframed}
\usepackage{nicefrac}
\usepackage{bm}
\usepackage{inconsolata}
\usepackage{setspace}

\setlist[enumerate]{label=\texttt{(\alph*)},itemsep=0pt,topsep=3pt}

\newcommand{\id}{\operatorname{Id}}

\newcommand{\diag}{\operatorname{diag}}

\newcommand{\sign}{\operatorname{sign}}
\newcommand{\Dmin}{\Delta_{2}}
\newcommand{\Dmax}{\Delta_{k}}
\newcommand{\exps}{e}

\usepackage[textsize=tiny]{todonotes} 

\newcommand{\exl}[1]{\stackrel{\mathclap{\text{\tiny \texttt{\color{red!60!black}(#1)}}}}}
\newcommand{\ex}[1]{\texttt{\color{red!60!black}(#1)}}
\newcommand{\arcsinh}{\operatorname{arcsinh}}

\newcommand{\E}{\mathbb E}

\newcommand{\bbP}{\mathbb P}

\newcommand{\cN}{\mathcal N}

\newcommand{\sF}{\mathcal F}

\newcommand{\norm}[1]{\left\Vert #1 \right\Vert}
\newcommand{\R}{\mathbb R}
\newcommand{\ip}[1]{\langle #1 \rangle}
\renewcommand{\d}[1]{\operatorname{d}\!#1}

\newcommand{\sind}{\bm 1}
\newcommand{\ones}{\bm 1}
\newcommand{\zeros}{\bm 0}

\newcommand{\Reg}{\operatorname{Reg}}

\theoremstyle{plain}
\newtheorem{theorem}{Theorem}
\newtheorem{lemma}[theorem]{Lemma}
\newtheorem{proposition}[theorem]{Proposition}

\theoremstyle{definition}
\newtheorem{definition}[theorem]{Definition}

\newtheorem{remark}[theorem]{Remark}

\crefname{example}{Example}{Examples}

\begin{document}
\begin{abstract}
We study a continuous-time diffusion approximation of policy gradient for $k$-armed stochastic bandits.
We prove that with a learning rate $\eta = O(\Delta^2/\log(n))$ the regret is $O(k \log(k) \log(n) / \eta)$ where $n$ is the horizon 
and $\Delta$ the minimum gap.
Moreover, we construct an instance with only logarithmically many arms for which the regret is linear unless $\eta = O(\Delta^2)$.
\end{abstract}

\maketitle

\section{Introduction}
Policy gradient is a classical and widely used reinforcement learning algorithm \citep{sutton1998reinforcement}.
Our focus is on the dynamics of policy gradient with the softmax policy class on Gaussian bandits.
Even in this simple setup, only the two-armed setting is really properly understood.
We opt for the unusual simplification of studying a continuous-time diffusion approximation
of the policy gradient algorithm.
We believe (but do not prove) that this is a high-quality approximation. This approach comes with a number of advantages.
Most straightforwardly, the randomness arising from the sampling of actions is removed, which slightly simplifies the analysis.
More significantly, by switching to continuous time we can exploit the vast literature on stochastic differential equations.
Needless to say, we hope the proof ideas can be generalised to discrete time eventually, but for now we are happy to settle with simpler analysis
and new means for understanding a delicate problem.

\paragraph{Basic notation}
We let $\ones$ and $\zeros$ denote vectors of all ones and zeros, respectively, and with the dimension always obvious from the context.
The standard basis vectors in $\R^k$ are $e_1,\ldots,e_k$ and
the identity matrix is $\id$. 
Given a vector $x$, $\diag(x)$ is the diagonal matrix with $x$ as its diagonal.
We let $\norm{\cdot}_p$ be the $p$-norm and for positive definite $A$, $\norm{x}_A = (x^\top A x)^{1/2}$.
We let $c, C > 0$ be suitably small/large absolute constants.

\paragraph{Bandit notation}
There are $k$ actions and the horizon is $n$. We assume purely for the sake of simplifying cluttered expressions that $n \geq \max(3, k)$. 
We will consider Gaussian rewards with mean $\mu \in \R^k$ and standard deviation $\sigma \in \R^k$.
We assume that
$1 \geq \mu_\star = \mu_1 > \mu_2 \geq \cdots \geq \mu_k \geq 0$ and that $\norm{\sigma}_\infty \leq 1$. 
We let $\Delta_a = \mu_1 - \mu_a$.
Given a vector $\theta \in \R^k$ we let $\pi(\theta)$ be the softmax policy with $\pi(\theta)_a \propto \exp(\theta_a)$.
In discrete time the learner sequentially selects actions $(A_t)_{t=1}^n$ over $n$ rounds and observes rewards $Y_t = \mu_{A_t} + \sigma_{A_t} \epsilon_t$
where $(\epsilon_t)_{t=1}^n$ is an independent sequence of standard Gaussian random variables.
The regret in the discrete time setting is
\begin{align*}
\Reg_n = n \mu_\star - \sum_{t=1}^n \mu_{A_t}\,,
\end{align*}
where $\mu_\star = \max_{a \in \{1,\ldots,k\}} \mu_a$.
Our main focus, however, is on a continuous time approximation of the discrete process.
Let $(B_t)$ be a standard $k$-dimensional Brownian motion on some probability space and $(\sF_t)$ be its natural filtration.
In the continuous-time bandit problem the learner selects a policy $(\pi_t)$, which is a process in the simplex adapted to the natural filtration associated with the process $(X_t)$ satisfying the
stochastic differential equation (SDE) given by
\begin{align*}
\d{X}_t = \diag(\pi_t) \mu \d{t} + \diag(\sqrt{\pi_t}) \Sigma^{1/2} \d{B}_t \,,
\end{align*}
where $\Sigma = \diag(\sigma^2)$.
The existence of a solution to $(X_t)$ might be hard to resolve for complex policies $(\pi_t)$ but will be straightforward for our choice of a softmax policy
updated using (continuous time) policy gradient.
The regret in continuous time is
$\Reg_n = \int_0^n R_t \d{t}$ with
$R_t = \mu_\star - \ip{\pi_t, \mu}$.
Note, $(B_t)$ is always used for the $k$-dimensional Brownian motion driving the process $(X_t)$. We use $(W_t)$ for \textit{some} standard $1$-dimensional Brownian motion,
but not always the same one.
The classical policy gradient algorithm is given in \cref{alg:pg}, while its continuous time approximation is in \cref{alg:cpg}.

\begin{algorithm}[h!]
\centering
\begin{tikzpicture}
\node[draw,fill=white!90!black,text width=12cm] (b) at (0,0) {
\begin{lstlisting}
args: learning rate $\eta > 0$
let $\theta_1 = \zeros$
for $t = 1$ to $\infty$:
  let $\pi_t = \pi(\theta_t)$
  sample $A_t$ from $\pi_t$ and observe $Y_t \sim \cN(\mu_{A_t}, \sigma_{A_t}^2)$
  update $\forall a\,\, \theta_{t+1,a} = \theta_{t,a} + \eta (\sind_{A_t = a} - \pi_{t,a}) Y_t$ 
\end{lstlisting}
};
\end{tikzpicture}
\caption{
Discrete time policy gradient.
}
\label{alg:pg}
\end{algorithm}

\begin{algorithm}[h!]
\centering
\begin{tikzpicture}
\node[draw,fill=white!90!black,text width=12cm] (b) at (0,0) {
\begin{lstlisting}
args: learning rate $\eta > 0$
let $\theta_0 = \zeros \in \R^k$
for $t \in [0, n]$:
  let $\pi_t = \pi(\theta_t)$
  observe $\d{X}_t = \diag(\pi_t) \mu \d{t} + \diag(\sqrt{\pi_t}) \Sigma^{1/2} \d{B}_t$
  update $\d{\theta}_t = \eta (\id - \pi_t \ones^\top) \d{X}_t$
\end{lstlisting}
};
\end{tikzpicture}
\caption{
Continuous time policy gradient.
}
\label{alg:cpg}
\end{algorithm}

\paragraph{Contribution}
We study the regret of \cref{alg:cpg}.
On the positive side, we show that if $\eta \leq c \Dmin^2 / \log(n)$,
then the regret is at most
\begin{align*}
\frac{C k\log(k)\log(n)}{\eta} \,.
\end{align*}
More negatively, we show that if $\eta = \Omega(\Dmin^2)$, then there exist instances with logarithmically many actions and a very large horizon where the regret is $\Omega(n \Dmin)$.

\paragraph{Related work}
\cite{mei2020global} study policy gradient in discrete time without noise and show $O(1/t)$ convergence.
A similar result in continuous time is provided by \cite{walton2020short}.
\cite{mei2023stochastic} also prove almost sure convergence, but with the assumption that the learning rate is $O(\Dmin^2 / k^{3/2})$.
\cite{mei2024small} show that policy gradient converges to a unique optimal action almost surely with any learning rate
and that asymptotically the regret is $O(\log(t))$ with non-specified problem-dependent constants.
\cite{baudry2025does}
show that setting the learning rate $\eta \leq \Dmin$ is needed for logarithmic regret and that this can fail when $\eta > \Dmin$.
When $k > 2$ they argue that $\eta$ must be $O(1/k)$ and prove that the regret is $O(k^2 \log(n) / \Dmin)$ when $\eta \sim \Dmin / k$ and
all suboptimal arms have the same gap.

\section{Policy gradient in continuous time}
In discrete time the classical policy gradient algorithm specified to $k$-armed bandits
is derived as follows.
The value function is $v(\theta) = \ip{\pi(\theta), \mu}$ and its gradient is
\begin{align}
v'(\theta) = (\diag(\pi) - \pi \pi^\top) \mu \,.
\label{eq:value}
\end{align}
This quantity is not observed, but is classically estimated in round $t$ using Monte-Carlo by
\begin{align*}
\widehat \nabla_t = (\id - \pi_t \ones^\top) e_{A_t} Y_t = (\id - \pi_t \ones^\top) e_{A_t} Y_t \,.
\end{align*}
The policy gradient update with learning rate $\eta$ is then
\begin{align*}
\theta_{t+1} = \theta_t + \eta \widehat \nabla_t \,.
\end{align*}
In continuous time there are two ways to proceed.
One option is to approximate the discrete time stochastic process $(\theta_t)$ by a continuous process with matching variance
and drift. Alternatively, the mechanism for deriving the policy gradient method can be repeated from first principles in continuous
time. \cref{eq:value} remains the same in continuous time. Because the drift of the continuous-time process $(X_t)$ is
$\diag(\pi_t) \mu$, an unbiased continuous-time estimator of $v'(\theta_t)$ is
$(\id - \pi_t \ones^\top) \d{X_t}$.
Then the continuous-time policy gradient algorithm is naturally defined by the stochastic process
\begin{align*}
\d{\theta_t} 
&= \eta (\diag(\pi_t) - \pi_t \pi_t^\top) \mu \d{t} + \eta \diag(\sqrt{\pi_t}) \Sigma^{1/2}\d{B}_t - \eta \pi_t \ip{\sqrt{\pi_t}, \Sigma^{1/2}\d{B}_t} \,. 
\end{align*}
This is equivalent to approximating the discrete time policy gradient algorithm by matching the mean and variance.

\paragraph{Elementary properties}
Before the serious analysis, we start with a collection of elementary results about the vectors $(\theta_t)$ and $(\pi_t)$ defined in
\cref{alg:cpg}.
The proofs either appear inline or in the appendix.

\begin{lemma}[\textsc{conservation}]\label{lem:conserve}
$\sum_{a=1}^k \theta_{t,a} = 0$ almost surely.
\end{lemma}

\begin{proof}
Substitute the definition of $\d{\theta}_t$ to show that $\ip{\ones, \d{\theta}_t} = 0$.
\end{proof}

The next lemma shows that with high probability $\theta_{t,a}$ never becomes too negative.
This is intuitive since by \cref{lem:conserve}, $\max_{b} \theta_{t,b} \geq 0$ so that $\pi_{t,a} \leq \exp(\theta_{t,a})$. Hence, if $\theta_{t,a}$ becomes very negative, then
$\pi_{t,a}$ becomes vanishingly small and the updates by policy gradient to $\theta_{t,a}$ become negligible.

\begin{lemma}\label{lem:bounds}
Suppose that $\eta \leq 1$. Then, for all actions $a \in \{1,\ldots,k\}$ and $\delta \in (0,1)$,
\begin{align*}
\bbP\left(\inf_{t \leq n} \theta_{t,a} \leq -\log\left(n/\delta\right)\right) \leq \delta \,.
\end{align*}
\end{lemma}

The standard behaviour of an effective bandit algorithm is that suboptimal arms are slowly eliminated as they are identified as being statistically
suboptimal. For randomised algorithms this roughly means that $1/\pi_{t,1}$ is a measure of the number of actions that still appear competitive. 
The following lemma relates this quantity to the magnitude of $\theta_{t,1}$ showing that as $\theta_{t,1}$ grows, the number of arms that appear competitive
reduces:

\begin{lemma}\label{lem:pi}
Suppose that $\theta_{t,a} \in [-\log(n/\delta), k \log(n/\delta)]$ for all $a$ and $\theta_{t,1} \geq \max_{a} \theta_{t,a} - 1$. Then
\begin{align*}
\frac{1}{\pi_{t,1}} \leq  \frac{6k \log(n/\delta)}{\theta_{t,1} + 1 + \log(n/\delta)} \,.
\end{align*}
\end{lemma}

The conditions in \cref{lem:pi} are interesting. The first holds with high probability by \cref{lem:conserve,lem:bounds}. Later we will prove the second holds with high probability
when the learning rate is sufficiently small and is equivalent to $\pi_{t,1} / \pi_{t,a} \geq 1/e$ for all $a$.

\section{Upper bounds}

To avoid clutter we assume that the variance matrix is the identity: $\Sigma = \id$. This corresponds to the assumption in the discrete time
setting that the reward distributions for all actions have unit variance. All results in this section still hold when $\Sigma \preceq \id$.
As a warmup, we start with the two-armed setting, which is straightforward and also well-understood in the discrete time setting \citep[Theorem 1]{baudry2025does}.
In fact, our proof follows theirs almost without modification.

\begin{proposition}\label{prop:two}
Suppose that $a = \Dmin / \eta > 1$. Then with $k = 2$,
\begin{align*}
\E[\Reg_n] 
\leq \frac{a}{2 \Dmin} \log\left(1 + \frac{2(a+1)n \Dmin^2}{a^2}\right) + \frac{a^2}{2(a-1) \Dmin} \,. 
\end{align*}
\end{proposition}

\begin{proof}
Abbreviate $\pi_t = \pi_{t,1}$.
Starting with the same regret decomposition as proposed by \cite{baudry2025does}, the regret is
\begin{align}
\Reg_n = \Dmin \int_0^n (1 - \pi_t) \d{t} 
= \Dmin \int_0^n \pi_t(1-\pi_t) \d{t} + \Dmin \int_0^n (1 - \pi_t)^2 \d{t}\,.
\label{eq:decomp}
\end{align}
Let $Z_t = \theta_{t,1} - \theta_{t,2}$ and
$M_t = \exp(b Z_t)$. By It\^o's rule,
\begin{align*}
\E[M_n] - 1
&= \E\left[\int_0^n \pi_t(1-\pi_t) M_t (2b \eta \Dmin + 2b^2 \eta^2)\d{t}\right] 
= \E\left[\int_0^n \pi_t(1-\pi_t) M_t \Dmin^2 (2b / a + 2b^2/a^2) \d{t}\right]\,.
\end{align*}
When $b = 1$, then $M_t = \pi_t / (1 - \pi_t)$ and in this case
\begin{align}
\E[M_n] = 1 + \E\left[\int_0^n \pi_t^2 \Dmin^2 (2/a + 2/a^2)\right] \leq 1 + \left(2/a + 2/a^2\right) n \Dmin^2  \,.
\label{eq:M1}
\end{align}
Therefore, since $(Z_t)$ has drift $2 \eta \Dmin \pi_t(1-\pi_t)$ and $Z_0 = 0$,
\begin{align}
\Dmin \E\left[\int_0^n \pi_t(1-\pi_t) \d{t}\right]
&= \frac{\E[Z_n]}{2\eta} 
\leq \frac{\log(\E[M_n])}{2\eta} 
\leq \frac{a}{2\Dmin} \log\left(1 + \frac{2(a+1)n\Dmin^2}{a^2} \right) \,,
\label{eq:Mn1}
\end{align}
where the inequalities follow from Jensen, \cref{eq:M1} and because $a = \Dmin / \eta$.
For the other term in the decomposition, let $b = -1$ so that $M_t = (1-\pi_t) / \pi_t$ and
\begin{align*}
\E[M_n] - 1 = \E\left[\int_0^n (1 - \pi_t)^2 \Dmin^2(2/a^2 - 2/a)\right] \,.
\end{align*}
Since $M_n$ is non-negative and $a > 1$,
\begin{align}
\Dmin \E\left[\int_0^n (1 - \pi_t)^2 \d{t}\right] \leq \frac{a^2}{2(a-1) \Dmin} \,.
\label{eq:Mn2}
\end{align}
Combining \cref{eq:Mn1,eq:Mn2} with \cref{eq:decomp} yields the bound.
\end{proof}

\begin{remark}
As $n \to \infty$ and by choosing $a \to 1$ very slowly one obtains the regret for the (non-anytime) algorithm
of $\E[\Reg_n] \sim \frac{\log(n)}{2 \Dmin}$, which appears to contradict the lower bound by \cite{LR85}.
The reason there is no contradiction is that the learning rate depends on $\Dmin$ and in this scenario the
classical lower bound does not apply and the $\frac{\log(n)}{2 \Dmin}$ regret is optimal \citep{garivier2016explore}.
\end{remark}

The argument above does not generalise to the situation where there are more than two actions.
Not only is the problem more difficult to analyse, the learning rate needs to be considerably smaller and the regret is correspondingly much worse.
We will prove the following:

\begin{theorem}\label{thm:upper}
Suppose that $\eta \leq \frac{\Dmin^2}{8\log(2n^2)}$, then $\E[\Reg_n] = O\left(\frac{k \log(k) \log(n)}{\eta} \right)$.
\end{theorem}

We start with a straightforward lemma governing the dynamics of $\log(\frac{\pi_{t,1}}{\pi_{t,a}})$ that reveals the main challenge when $k > 2$.

\begin{lemma}\label{lem:diff}
Let $Z_{t,a} = \theta_{t,1} - \theta_{t,a}$. Then
\begin{align*}
\d{Z_{t,a}}
&= \eta \left[\pi_{t,a} \Delta_a + (\pi_{t,1} - \pi_{t,a}) R_t\right] \d{t} + \eta \sqrt{\pi_{t,1} + \pi_{t,a} - (\pi_{t,1} - \pi_{t,a})^2} \d{W}_t\,,
\end{align*}
where $(W_t)$ is a standard Brownian motion. 
\end{lemma}

The expression in \cref{lem:diff} is revealing. If we are to prove the algorithm is learning, then we must show that $Z_{t,a}$ increases with time for $a \neq 1$.
Naturally this will be easiest to prove if the SDE associated with $Z_{t,a} = \theta_{t,1} - \theta_{t,a}$ has positive drift and well-controlled noise.
Regrettably, however, this is not the case. The drift $\eta \pi_{t,a} \Delta_a + \eta (\pi_{t,1} - \pi_{t,a}) R_t$ can be negative if $\pi_{t,1} < \pi_{t,a}$ 
and the regret is large. If there were no noise, then an elementary analysis of the differential equation would show that the drift is always positive
because at initialisation $\pi_{t,1} \geq \pi_{t,a}$ and the positive drift ensures this stays true throughout the trajectory. 
The following lemma shows that when $\eta$ is suitably small, then $Z_{t,a}$ can be lower-bounded with high probability.

\begin{lemma}\label{lem:eta}
Suppose that $\eta \leq \frac{\Dmin^2}{8\log(2n/\delta)}$, then $\bbP(\inf_{t \leq n} Z_{t,a} \leq -\Dmin/2) \leq \delta$.
\end{lemma}

\begin{proof}
Abbreviate $Z_t = Z_{t,a}$ and let $\tau = \inf\{t : Z_t = -\Dmin/2\}$ and
$s(t) = \eta^2 \int_0^t (\pi_{u,1} + \pi_{u,a} - (\pi_{u,1} - \pi_{u,a})^2) \d{u}$ and $t(\cdot) = s^{-1}(\cdot)$ be its inverse.
Lastly, let $U_{s(t)} = Z_t$. 
By definition, $\inf_{t \leq n} Z_t$ has the same law as $\inf_{s \leq s(n)} U_s$. Moreover, when $U_s \geq -\Dmin/2$, then
\begin{align*}
\d{U_s} 
&= \frac{1}{\eta} \left[\frac{\pi_{t(s),a} \Delta_a + \pi_{t(s),a}(\exp(U_s) - 1) R_{t(s)}}{\pi_{t(s),1} + \pi_{t(s),a} - (\pi_{t(s),1} - \pi_{t(s),a})^2}\right] \d{s} + \d{W}_s \\
&\geq \frac{1}{2\eta} \left[\frac{\pi_{t(s),a} \Delta_a}{\pi_{t(s),1} + \pi_{t(s),a}}\right] \d{s} + \d{W}_s 
= \frac{\Delta_a}{2 \eta (\exp(U_s) + 1)} \d{s} + \d{W}_s \,.
\end{align*}
where the inequality holds because $R_{t(s)} \leq 1$ and $\exp(U_s) - 1 \geq \exp(-\Dmin/2) - 1 \geq -\Dmin/2 \geq -\Delta_a/2$, which means the numerator is positive.
Looking at this relation, when the learning rate is small relative to $\Dmin$, the process has considerable positive drift when $U_s \in [-\Dmin/2,1]$.
Formally, since $s(n) \leq n$ and by a comparison to the SDE in \cref{prop:bm-less-drift},
$\bbP\left(\inf \{U_s : 0 \leq s \leq s(n)\} \leq -\Dmin/2\right) \leq \delta$.
\end{proof}

We are now in a position to prove \cref{thm:upper}.

\begin{proof}[Proof of \cref{thm:upper}]
Let $\delta = 1/n$ and $Z_{t,a} = \theta_{t,1} - \theta_{t,a}$ and define a stopping time $\tau$ by
\begin{align*}
\tau = \inf\left\{t \leq n : \min_{a \in 2,\ldots,k} Z_{t,a} \leq -\Dmin/2 \text{ or } \min_{a \in 1,\ldots,k}\theta_{t,a} \leq -\log(n / \delta)\right\}\,.
\end{align*}
By \cref{lem:bounds,lem:eta} and a union bound, with probability at least $1 - 2k \delta$, $\tau = n$.
For $t \leq \tau$, \cref{lem:conserve} shows that $-\log(n/\delta) \leq \theta_{t,1} = -\sum_{a=2}^k \theta_{t,a} \leq k \log(n/\delta)$.
For $u > -\log(n/\delta)$, let $\psi'(u) = \frac{6k \log(n/\delta)}{u + 1 + \log(n/\delta)}$, which is the function that appeared in \cref{lem:pi}.
Then let
\begin{align*}
\psi(u) = \int_0^u \psi'(v) \d{v} = 6k \log(n/\delta) \log\left(\frac{u + 1 + \log(n/\delta)}{1 + \log(n/\delta)}\right) \,.
\end{align*}
Later we will use the fact that
\begin{align}
\psi(k \log(n/\delta)) \leq 6k \log(n/\delta)  \log(1+k) \,.
\label{eq:psi}
\end{align}
Moreover, for $u > -\log(n/\delta)$,
\begin{align}
\psi''(u) = -\frac{6 k \log(n/\delta)}{(u + 1 + \log(n/\delta))^2} \geq -\psi'(u)\,.
\label{eq:Psi2}
\end{align}
Let $\sigma_t^2 = \eta^2 \pi_{t,1} (1 - \pi_{t,1})$. By definition $\d{\theta_{t,1}} = \eta \pi_{t,1} R_t \d{t} + \sigma_t \d{W}_t$
and hence by It\^o's formula,
\begin{align*}
\d{\psi(\theta_{t,1})} 
&= \psi'(\theta_{t,1}) \d{\theta_{t,1}} + \frac{1}{2} \psi''(\theta_{t,1}) \sigma_t^2 \d{t} \\
&\exl{a}\geq \left[\eta \psi'(\theta_{t,1}) \pi_{t,1} R_t - \frac{\eta^2 \psi'(\theta_{t,1})\pi_{t,1}(1-\pi_{t,1})}{2} \right] \d{t} + \psi'(\theta_{t,1}) \sigma_t \d{W}_t \\
&\exl{b}\geq \left[\eta \psi'(\theta_{t,1}) \pi_{t,1} R_t - \frac{\eta^2 \psi'(\theta_{t,1})\pi_{t,1} R_t}{2 \Dmin} \right] \d{t} + \psi'(\theta_{t,1}) \sigma_t \d{W}_t \\
&\exl{c}\geq \frac{1}{2} \eta \psi'(\theta_{t,1}) \pi_{t,1} R_t \d{t} + \psi'(\theta_{t,1}) \sigma_t \d{W}_t \\ 
&\exl{d}\geq \frac{1}{2} \eta R_t \d{t} + \psi'(\theta_{t,1}) \sigma_t \d{W}_t \,. 
\end{align*}
where \ex{a} follows from \cref{eq:Psi2},
\ex{b} since $R_t \geq (1 - \pi_{t,1}) \Delta_2$,
\ex{c} since $\eta \leq \Dmin$ and \ex{d} follows from \cref{lem:pi}.
Taking expectations and rearranging in combination with \cref{eq:psi} shows that
\begin{align*}
\E[\Reg_\tau] 
&\leq \frac{2}{\eta}\E[\psi(\theta_{\tau,1})] 
\leq \frac{12k \log(n/\delta) \log(1+k)}{\eta} \,.
\end{align*}
Since $\bbP(\tau < n) \leq 2k\delta = 2k/n$, $\E[\Reg_n] \leq 2k + \E[\Reg_\tau]$ and the argument is complete.
\end{proof}

\begin{remark}
The analysis in \cref{lem:eta} can be improved by refining the bound $R_t \leq 1$ to $R_t \leq \Delta_k$. 
With this change one obtains the same bound as \cref{thm:upper} but the condition on $\eta$ can be relaxed to
$\eta = \tilde O(\Dmin^2/\Dmax)$. Consequentially, when all actions have the same suboptimality gap and $\eta = \frac{c\Dmin}{\log(n)}$, then
\begin{align*}
\E[\Reg_n] = O\left(\frac{k \log(k) \log(n)^2}{\Dmin}\right)\,.
\end{align*}
\end{remark}

\section{Lower bound}\label{sec:lower}

\cref{prop:two} shows that when there are only two actions, then setting $\eta$ just slightly smaller than $\Dmin$ leads to near-optimality of policy gradient.
The next construction shows that when there are more than two actions the correct choice of $\eta$ can be as small as $\Dmin^2$ and in general 
no choice leads to near-optimal regret in the sense of being
remotely close to the lower bound by \cite{LR85}.
The fundamental difference is that when $k = 2$, the drift of $\theta_{t,1} - \theta_{t,2}$ is always positive. When there are more arms, this only holds on trajectories
that are well-behaved in the sense that $\theta_{t,1} - \theta_{t,a}$ never gets too negative for $a \neq 1$.

\paragraph{Lower bound construction}
The poor behaviour of policy gradient is most dramatic when $k > 2$ and $\mu$ is such that $\Delta = (0, \Dmin, 1,\ldots,1)$ with $\Dmin$ close to zero.
The following arguments can be made to work with $\Sigma = \id$ but are simplified considerably by instead letting $\Sigma = \diag(e_1 + e_2)$ so that the only noise arises from the first two actions.
What happens in this scenario is that for a long time actions $1$ and $2$ are nearly statistically indistinguishable.
Since actions $a > 2$ are very suboptimal, $\theta_{t,a}$ for $a > 2$ decreases rapidly and $\theta_{t,1}$ and $\theta_{t,2}$ increase.
But they do not increase together at the same rate. Rather, unless the learning rate is very small, the dynamics of policy gradient effectively `picks a winner' in $\{1,2\}$ at random
and by the time $\pi_{t,a}$ is negligible for $a > 2$, it turns out that either $\pi_{t,1}$ or $\pi_{t,2}$ is nearly $1$.
From then on policy gradient behaves as if the bandit has two actions, but with a potentially very disadvantageous initialisation.

\begin{theorem}\label{thm:lower}
When $\Delta = (0, \Dmin, 1,\ldots,1)$ and $\Sigma = \diag(1, 1, 0, \ldots, 0)$ and $k \geq C \log(n / \Dmin)$ and 
learning rate bounded in $C \Dmin^2 \leq \eta \leq c / k$,
the regret of continuous policy gradient (\cref{alg:cpg}) satisfies $\E[\Reg_n] = \Omega(n \Dmin)$.
\end{theorem}

\begin{remark}
The assumption that $\eta \leq c/k$ is for simplicity in the analysis and can be removed at the price of a slightly more complicated proof.
\end{remark}

In full gory detail the proof is embarrassingly intricate. 
We try to present the high-level argument in the main body and defer the fiddly computations to the appendix.
Alternatively, you can skip all the analysis and look at \cref{fig:lower} in \cref{app:exp}, which illustrates the striking behaviour of the discrete time policy gradient on
the instance in \cref{thm:lower}. You may also wonder if the lower bound might hold with $k = 3$. We believe the answer is no, which is partially supported by
the results in \cref{fig:lower3}.

\newcommand{\pibad}[1]{\pi_{#1,3+}}
\begin{proof}[Proof of \cref{thm:lower}]
Abbreviate $m = k-2$ and $\pibad{t} = \sum_{a=3}^k \pi_{t,a}$ to be the probability assigned to all but the first two actions.
Our argument revolves around showing that with constant probability $\theta_{t,2}$ grows much faster than $\theta_{t,1}$ for the initial period when 
$\pibad{t}$ is decreasing rapidly.
Let $S_t = \theta_{t,1} + \theta_{t,2}$ and $Z_t = \theta_{t,1} - \theta_{t,2}$
and $C_t = \eta \pibad{t}(1 - \pibad{t})$ and $s(t) = \int_0^t C_u \d{u}$ and $s \mapsto t(s)$ be its inverse.

\begin{definition}\label{def:tau}
Let $\epsilon = 2/m$ and
define a stopping time $\tau$ as the first time that any of the following holds:
\begin{enumerate}
\item $\tau = n$; or
\item $S_\tau \notin ((1 - \Dmin) s(\tau) - 1, s(\tau) + 1)$; or
\item $Z_\tau \geq 2 \arcsinh\left(\sqrt{\eta} (\exps^{-s(\tau)}/4 - 1/16) \exps^{(1-\epsilon)s(\tau)/2} \right)$; or
\item $s(\tau) = s_{\max} \triangleq \frac{1}{1-\epsilon}[\log(400/\eta) + \log(n)]$.
\end{enumerate}
\end{definition}

By our assumption that the variance only arises from the first two actions it follows that $\theta_{t,a} = \theta_{t,b}$ for all $t$ and $a, b \geq 3$,
which by \cref{lem:conserve} and the definition $S_t = \theta_{t,1} + \theta_{t,2}$ means that $\theta_{t,a} =-S_t / m$ and
\begin{align*}
\frac{\pi_{t,1}}{\pibad{t}}
&= \frac{1}{m} \exp\left(\frac{S_t + Z_t}{2} + \frac{S_t}{m}\right) \triangleq G_t \,. 
\end{align*}
Note that if $t \leq \tau$, then by the inequality $\arcsinh(x) \leq -\log(-2x)$, 
\begin{align}
Z_t \leq \min\left(1, \log\left(\frac{400}{\eta}\right) - (1 - \epsilon)s(t) \right)\,.
\label{eq:Z}
\end{align}

\paragraph{Step 1: Dynamics and intuition}
By substituting the definitions and crude bounding we obtain the following:
\begin{align}
\d{S_t} 
&=  \big[1 - \frac{\pi_{t,2}}{  1-\pibad{t}} \Delta_2\big] C_t \d{t} + \sqrt{\eta \pibad{t} C_t} \d{W}_t \quad \text{and} \label{eq:St}\\
\d{Z_t}
&\leq \big[\Dmin(1 + 2 G_t) + \tanh(Z_t/2) \big] C_t \d{t} + \sqrt{\eta \sigma_t^2 C_t} \d{W}_t \label{eq:Zt} \,,
\end{align}
where
\begin{align}
1 - \tanh(Z_t/2)^2 \leq \sigma_t^2 \leq 1 + 4 G_t \,.
\label{eq:sigma}
\end{align}
The tedious calculation establishing the above is deferred to the last step.
You will surely notice that these SDEs have been massaged so that the drift terms depend linearly on $C_t$ and the diffusion terms depend linearly on $\sqrt{C_t}$,
which suggests we might make progress by changing time using the clock $s(t) = \int_0^t C_u \d{u}$.
Rather than moving immediately to the formalities, let us spend a moment to argue informally.
The SDE in \cref{eq:St} is very well-behaved after a time change.
Since $\Dmin$ is small and $\pi_{t,2} \leq \pi_{t,1} + \pi_{t,2} = 1 - \pibad{t}$, we have
\begin{align*}
\d{S_{t(s)}} \approx \d{s} + \sqrt{\eta \pibad{t(s)}} \d{W_s} \,,
\end{align*}
which implies that with high probability $S_{t(s)} \approx s$ for all $s \leq s_{\max}$.
The SDE in \cref{eq:Zt} is much more complicated. At initialisation $G_t = 1/m$ is relatively small, so that
\begin{align*}
\d{Z}_{t(s)} \approx [\Dmin + \tanh(Z_{t(s)}/2)] \d{s} + \sqrt{\eta} \d{W_s} \,.
\end{align*}
When $|z|$ is small, then $\tanh(z/2) \approx z/2$. Hence, when $|Z_{t(s)}|$ is small, a reasonable approximation for a solution (see \cref{rem:factors}) is
\begin{align*}
Z_{t(s)} \approx 2(\exp(s/2) - 1) \Dmin + \sqrt{\eta} \exp(s/2) \int_0^s \exp(-u/2) \d{W}_u \,,
\end{align*}
which means that $Z_{t(s)}$ has approximate law $\cN(2 (\exp(s/2) - 1) \Dmin, \eta (\exp(s) - 1))$.
In particular, provided that $\sqrt{\eta}$ is suitably larger than $\Dmin$, then the noise dominates and with roughly constant probability $|Z_{t(s)}| \approx \exp(s/2) \sqrt{\eta}$
and $\sign(Z_{t(s)})$ is close to uniform on $\{-1,1\}$.
In particular, if $s \approx \log(1/\eta)$, then $|Z_{t(s)}| \approx 1$. At this point $\tanh(Z_{t(s)}/2) \approx \sign(Z_{t(s)})$ 
and the drift term is roughly constant and depending on the sign the process drifts up or down at a roughly linear rate.
Hence a reasonable ansatz is that $|Z_{t(s)}| \approx s$ for large $s$ with the sign roughly uniform when $\sqrt{\eta}$ is sufficiently larger than $\Dmin$.
This turns out to be the correct rate. 
As a consequence, we can expect that $S_t \approx s(t)$ and $Z_t \approx -s(t)$ with roughly constant probability. 
Hence $\theta_{t,1} = (S_t + Z_t)/2 \approx 0$ while $\theta_{t,2} = (S_t - Z_t)/2 \approx s(t)$, which at minimum implies that $\theta_{t,1}$ does not grow much larger than $\theta_{t,2}$
and implies linear regret.
The challenge is we need a sample path result and the additional terms in \cref{eq:Zt} need to be handled carefully.

\begin{remark}
Handling the additional terms is unsurprisingly non-trivial. Indeed, they are what causes the argument above to fail asymptotically, which it must because for any learning rate
policy gradient does learn asymptotically.
\end{remark}

\paragraph{Step 2: Formal details}

The following proposition formalises the ansatz in the previous step.
The proof is based on a technical comparison to the standard linear SDE and is deferred to \cref{app:sde}.

\begin{proposition}\label{prop:sde}
With probability at least $c$ either $\tau = n$ or $s(\tau) = s_{\max}$.
\end{proposition}

Hence it suffices to show that on this event $\int_0^n R_t \d{t} \geq c n \Dmin$.
There are two cases. When $\tau = n$, then $Z_t < 1$ for all $t \leq n$ and $\pi_{t,1} \leq e \pi_{t,2} \leq e (1 - \pi_{t,1})$ and hence $\pi_{t,1} \leq e/(1+e)$, from which it follows
that $\Reg_n = \Omega(n \Dmin)$.
Suppose instead that $\tau < n$, which means that $s(\tau) = s_{\max}$.
Therefore, by \cref{eq:Z},
\begin{align*}
Z_\tau 
\leq \log\left(\frac{400}{\eta}\right) - (1 - \epsilon) s_{\max} 
= -\log(n) \,.
\end{align*}
Similarly,
\begin{align*}
S_\tau \geq (1 - \Dmin) s_{\max} - 1 \geq (1 - \epsilon) s_{\max} - 1 = \log\left(\frac{400n}{\eta}\right) - 1 \geq \log(n/\eta) \,.
\end{align*}
Therefore $\pi_{\tau,1} \leq \exp(Z_\tau) \leq \frac{1}{n}$ and since $\eta \leq c/k \leq c/m$,
\begin{align*}
\pibad{\tau}
\leq m\exp\left(-\frac{S_\tau}{m} - \frac{S_\tau - Z_\tau}{2}\right)
\leq \frac{m \eta}{n} \leq 1/n \,.
\end{align*}
Hence $\pi_{t,2} \geq 1 - 2/n$ and from this a straightforward argument shows that with high probability the algorithm never recovers:
\begin{align*}
\bbP\left(\sup_{t \geq \tau} \pi_{t,1} \geq 1/2 \,\bigg|\, \sF_\tau\right) \leq 1/2 \,.
\end{align*}
Since $Z_t < 1$ for $t \leq \tau$, it follows that with constant probability $\pi_{t,1} \leq 1/2$ for all $t \leq n$ and therefore
$\E[\Reg_n] = \Omega(n \Dmin)$ as required.

\paragraph{Step 4: The calculations}
We promised the concrete calculations. By definition
\begin{align*}
\d{\theta_{t,1}}
&= \eta \pi_{t,1} R_t \d{t} + \eta\left[\sqrt{\pi_{t,1}} \d{B}_{t,1} - \pi_{t,1}^{3/2} \d{B}_{t,1} - \pi_{t,1} \sqrt{\pi_{t,2}} \d{B}_{t,2}\right] \text{ and } \\
\d{\theta_{t,2}}
&= \eta \pi_{t,2} \left[R_t - \Delta_2\right] \d{t} + \eta \left[\sqrt{\pi_{t,2}} \d{B}_{t,2} - \pi_{t,2}^{3/2} \d{B}_{t,2} - \pi_{t,2} \sqrt{\pi_{t,1}} \d{B}_{t,1}\right] \,.
\end{align*}
The calculation for $\d{S_t}$ follows by substituting the definitions.
For $\d{Z_t}$, we have
\begin{align*}
\d{Z_t}
&= \eta\left[\pi_{t,2} \Delta_2 + (\pi_{t,1} - \pi_{t,2}) R_t\right] \d{t} + \eta \sqrt{\pi_{t,1} + \pi_{t,2} - (\pi_{t,1} - \pi_{t,2})^2} \d{W}_t \\
&= \left[\alpha_t \Dmin + \tanh(Z_t/2)\right] C_t \d{t} + \sqrt{\eta \sigma_t^2 C_t} \d{W}_t\,,
\end{align*}
with $(\alpha_t)$ and $(\sigma_t)$ defined by
\begin{align*}
\alpha_t &= \frac{\pi_{t,2}(1 + \pi_{t,1} - \pi_{t,2})}{\pibad{t}(1 - \pibad{t})} \\
\sigma_t^2 &= 1 + \frac{4\pi_{t,1} \pi_{t,2}}{\pibad{t}(1 - \pibad{t})} - (1 - \pibad{t}) \left(\frac{\pi_{t,1} - \pi_{t,2}}{\pi_{t,1} + \pi_{t,2}}\right)^2  \,.
\end{align*}
Since $\pi_{t,2} \leq \pi_{t,1} + \pi_{t,2} = 1 - \pibad{t}$,
\begin{align*}
\alpha_t = \frac{\pi_{t,2}(2\pi_{t,1} + \pibad{t})}{\pibad{t}(1 - \pibad{t})} \leq 1 + 2 G_t \,. 
\end{align*}
Moreover, $\frac{\pi_{t,1} - \pi_{t,2}}{\pi_{t,1} + \pi_{t,2}} = \tanh(Z_t/2)$ so that
$1 - \tanh(Z_t/2)^2 \leq \sigma_t^2 \leq 1 + 4 G_t$.
\end{proof}

\section{Discussion}\label{sec:disc}

\paragraph{Continuous time vs discrete time}
When the learning rate is small, the continuous time and discrete time processes are extremely similar.
Even for larger learning rates, the noise terms contract asymptotically in both cases, suggesting that the diffusion approximation may be good
for any learning rate once the algorithm reaches the asymptotic regime.
We believe our proof technique for the upper bound is likely to translate to the discrete time setting, probably with only a little tedium.
Proving the lower bound in discrete time is probably more challenging and arguably less worthwhile.

\paragraph{Refining the lower bound}
Our upper bound is $\tilde O(k / \Dmin^2)$. 
The lower bound shows that this is not in general improvable when $k$ is logarithmic.
An interesting direction is to extend the lower bound construction to introduce a linear dependence on $k$.

\paragraph{$\bm{k}$-dependence}
\cite{baudry2025does} constructed a lower bound showing that the learning rate needs to be $O(1/k)$ and speculated that the correct learning rate might be around $\Dmin / k$.
Our upper bound holds without assuming that $\eta$ is less than $O(1/k)$, suggesting that this is a transient requirement where the continuous time and discrete time setups are
not comparable. Moreover, our lower bound shows that in certain cases the learning rate needs to be $O(\Dmin^2)$ in order to achieve sublinear regret.

\paragraph{Logarithmic factors}
The upper bound in \cref{thm:upper} relies on the assumption that $\eta \leq c \Dmin^2 / \log(n)$. One may wonder if this logarithmic term in the denominator can be dropped.
The argument in \cref{lem:eta} is quite crude, so there may be room for improvement there.

\bibliographystyle{abbrvnat}
\bibliography{all}

\appendix

\crefalias{section}{appendix}

\section{Technical inequalities}\label{app:tech}

\begin{proposition}\label{prop:bm-drift}
Suppose that $X_0 = 0$ and $\d{X}_t = a \d{t} + \d{B_t}$ with $a > 0$. Then
\begin{align*}
\bbP\left(\inf_{t \geq 0} X_t \leq -\epsilon\right) \leq \exp(-2a\epsilon) \,.
\end{align*}
\end{proposition}

\begin{proof}
Use the fact that $M_t = \exp(-2a X_t)$ is a martingale and Doob's maximal inequality.
\end{proof}

\begin{proposition}\label{prop:bm-less-drift}
Suppose that $X_0 = 0$ and $\d{X}_t = \frac{a}{\exp(X_t) + 1} \d{t} + \d{B}_t$. Then
\begin{align*}
\bbP\left(\inf\{X_t : 0 \leq t \leq n\} \leq -\epsilon\right) \leq (1+\sqrt{n}/2) \exp\left(-\frac{2a\epsilon}{\exps+1}\right)\,.
\end{align*}
In particular, if $a \geq \frac{\exps+1}{2\epsilon} \log((1+\sqrt{n}/2)/\delta)$, then $\bbP(\inf\{X_t : 0 \leq t \leq n\} \leq-\epsilon) \leq \delta$. 
\end{proposition}

\begin{proof}
Obviously the drift is positive everywhere and
when $X_t \in (-\infty, 1]$ the drift term is at least $a/(e+1)$.
Hence, letting $\tau = \inf\{t : X_t \in \{-\epsilon, 1\}\}$, it holds that
$\bbP(X_\tau =-\epsilon) \leq \exp(-2a\epsilon/(\exps+1))$.
Moreover, since $X_t$ has positive drift everywhere, the number of down-crossings $D_n$ of the interval $[0,1]$ is bounded in expectation by
the number of down-crossings of a Brownian motion in the same time interval, which by Doob's down-crossing lemma is at most
$\E[D_n] \leq \E[\max(0, B_n)] \leq \sqrt{n}/2$.
Therefore
$\bbP\left(\inf\{X_t : 0 \leq t \leq n\} \leq -\epsilon\right) 
\leq (1 + \sqrt{n}/2) \exp(-(2a\epsilon)(\exps+1))$.
\end{proof}

\section{Proof of Lemma~\ref{lem:bounds}}

Let $M_t = \exp(-\theta_{t,a})$. 
By \cref{lem:conserve} there exists a $b$ such that $\theta_{t,b} \geq 0$ and hence $M_t \leq \exp(\theta_{t,b} - \theta_{t,a}) = \pi_{t,b} / \pi_{t,a} \leq 1/\pi_{t,a}$.
Therefore, by It\^o's formula
\begin{align*}
\d{M_t} 
&= M_t \left[\eta \pi_{t,a} (\bar \mu_t - \mu_a) + \frac{\eta^2 \pi_{t,a}(1 - \pi_{t,a})}{2}\right] \d{t} + \eta M_t \sqrt{\pi_{t,a}(1-\pi_{t,a})} \d{W}_t \\
&\leq \left[\eta + \frac{\eta^2}{2}\right] \d{t} + \eta M_t \sqrt{\pi_{t,a}(1 - \pi_{t,a})} \d{W_t}\,.
\end{align*}
Therefore, since $M_0 = 1$, by Markov's inequality 
\begin{align*}
\bbP\left(\sup_{t \leq n} M_t \geq \frac{1}{\delta} + \frac{n}{\delta} \left[\eta + \frac{\eta^2}{2}\right]\right) \leq \delta \,,
\end{align*}
which after naively simplifying $\eta \leq c$ implies that with probability at least $1 - \delta$, $\theta_{t,a} \geq -\log\left(n/\delta\right)$. 

\section{Proof of Lemma~\ref{lem:pi}}

Suppose that $a \leq 0 \leq b$ and $K = \{x \in [a, b]^m : \sum_{i=1}^m x_i = 0\}$
and let $F$ be the set of measurable functions from $[0,m] \to [a, b]$ with $\int_0^m f(x) \d{x} = 0$. Then
\begin{align*}
\max_{x \in K} \norm{\exp(x)}_1 &\leq \max_{f \in F} \int_0^m \exp(f(x)) \d{x} = m \left[\frac{b}{b-a} \exp(a) - \frac{a}{b-a} \exp(b)\right]  \,.
\end{align*}
To reduce clutter we drop the time indices. 
Hence, by the previous display with $b = \theta_1 + 1$ and $a = -\log(n/\delta)$,
\begin{align*}
\norm{\exp(\theta)}_1
&\leq k\left[\frac{\theta_1+1}{\theta_1+1 + \log(n/\delta)} \left(\frac{\delta}{n}\right) + \frac{\log(n/\delta)}{\theta_1 + 1 + \log(n/\delta)} \exp(\theta_1+1)\right] \\
&\leq \frac{k\delta}{n} + \frac{k \log(n/\delta)}{\theta_1 + 1 + \log(n/\delta)} \exp(\theta_1  + 1) \,.
\end{align*}
Since $\theta_1 \geq \theta_a - 1$ and $\sum_{a=1}^k \theta_a = 0$ it follows that $\theta_1 \geq -1$.
Therefore
\begin{align*}
\frac{1}{\pi_1} = \norm{\exp(\theta)}_1 \exp(-\theta_1) \leq \exps\left[\frac{k\delta}{n} + \frac{k \log(n/\delta)}{\theta_1 + 1 + \log(n/\delta)}\right]\,.
\end{align*}
The result now follows by crudely simplifying the constants.

\section{Proof of Lemma~\ref{lem:diff}}

Let $\d{Z_{t,a}} = b_t \d{t} + \ip{\sigma_t, \d{B}_t}$.
By definition $\d{Z_{t,a}} = \ip{e_1 - e_a, \d{\theta}_t}$
Substituting the definition of $\d{\theta}_t$, the drift is
\begin{align*}
\ip{e_1 - e_a, (\diag(\pi_t) - \pi_t \pi_t^\top) \mu}
&= \pi_{t,1} (\mu_1 - \bar \mu_t) - \pi_{t,a} (\mu_a - \bar \mu_t) \\
&= \pi_{t,1} R_t - \pi_{t,a} (R_t - \Delta_a) \\
&= \pi_{t,a} \Delta_a + (\pi_{t,1} - \pi_{t,a}) R_t \,.
\end{align*}
Similarly, the diffusion coefficient is
$\eta \norm{e_1 - e_a}_{\diag(\pi_t) - \pi_t \pi_t^\top} = \eta \sqrt{\pi_{t,1} + \pi_{t,a} - (\pi_{t,1} - \pi_{t,a})^2}$.

\section{Proof of Proposition~\ref{prop:sde}}\label{app:sde}

The analysis of $(S_t)$ is straightforward.
We have
\begin{align*}
\d{S}_{t(s)} \leq \d{s} + \sqrt{\eta \pibad{t(s)}} \d{W_s}\,.
\end{align*}
Hence, $S_{t(s)} \leq s + \sqrt{\eta} W_s$ and by the maximal inequality, with probability at least $1 - \delta$ for all $t$ such that $s(t) \leq s_{\max}$,
\begin{align*}
S_{t(s)} \leq s + \sqrt{2\eta s_{\max} \log(1/\delta)} < s + 1\,.
\end{align*}
Similarly, $\d{S}_{t(s)} \geq (1 - \Dmin) \d{s} + \sqrt{\eta \pibad{t(s)}} \d{W_s}$ and with probability at least $1 - \delta$, for all $t$ such that $s(t) \leq s_{\max}$,
\begin{align*}
S_{t(s)} > (1 - \Dmin) s - 1\,.
\end{align*}
The last two parts involve $(Z_t)$ and are more challenging.

\begin{remark}[\textsc{method of integrating factors}]\label{rem:factors}
Remember that the linear SDE with $X_0 = 0$ and
$\d{X}_t = \left[A_t + B_t X_t\right] \d{t} + C_t \d{W}_t$ has the solution
\begin{align*}
X_t = \exps^{\beta(t)} \left[\int_0^t \exps^{-\beta(u)} A_u \d{u} + \int_0^t \exps^{-\beta(u)} C_u \d{W}_u\right]
\text{ with } \beta(t) = \int_0^t B_u \d{u} \,.
\end{align*}
\end{remark}

Recall the identities $\sinh'(x) = \cosh(x)$ and $\cosh'(x) = \sinh(x)$ and $\cosh(x) = \sqrt{1 + \sinh^2(x)}$.
Let $U_t = \sinh(Z_t/2)$.
By It\^o's formula,
\begin{align*}
\d{U}_t 
&\leq \frac{C_t}{2} \left[\Dmin(1 + 2G_t) \sqrt{1 + U_t^2} + U_t + \frac{\eta(1 + 4 G_t) U_t}{4}\right] \d{t} 
+ \frac{1}{2}\sqrt{\eta \sigma_t^2 (1 + U_t^2) C_t} \d{W}_t \,. 
\end{align*}
Let us start by upper bounding the drift term. Suppose that $t \leq \tau$. Then, $U_t \leq \sqrt{\eta}$.
Hence,
\begin{align*}
\textrm{Drift}_t 
&= \frac{C_t}{2} \left[\Dmin(1 + 2G_t) \sqrt{1 + U_t^2} + U_t + \frac{\eta(1 + 4 G_t) U_t}{4}\right] \\
&\exl{a}\leq \frac{C_t}{2} \left[\Dmin(1 + 2 G_t) (1 + \sqrt{2}) + U_t\left(1 - \Dmin(1 + 2 G_t)\sind_{U_t<0}\right) + \frac{\eta(1 + 4 G_t)}{4}\right] \\
&\exl{b}\leq \frac{C_t}{2} \left[6(\Dmin + \eta) (1 + G_t) + U_t\left(1 - \Dmin(1 + 2 G_t)\sind_{U_t<0}\right)\right] \\
&\exl{c}\leq \frac{C_t}{2} \left[6(\Dmin + \eta) (1 + G_t) + U_t\left(1 - \Dmin\left(1 + \frac{400}{m \sqrt{\eta}}\right) \sind_{U_t<0}\right)\right] \\
&\exl{d}\leq C_t \left[3\left(\Dmin + \eta + \epsilon \sqrt{\eta}\right) (1 + G_t)  + \frac{U_t\left(1 - \epsilon \right)}{2}\right] \\
&\exl{e}\leq C_t \left[\frac{\sqrt{\eta}}{96} (1 + G_t)  + \frac{U_t\left(1 - \epsilon \right)}{2}\right] \,. 
\end{align*}
where \ex{a} follows from the inequality $\sqrt{1+u^2} \leq 1 + \sqrt{2} - u \sind_{u < 0}$ for $u \leq 1$.
\ex{b} by direct simplification.
\ex{c} by the bound on $G_t$ below (\cref{eq:G2}).
\ex{d} since $\sqrt{\eta} \geq C \Dmin$ and $\Dmin \leq 1/m$ and by the definition of $\epsilon = 2/m$.
\ex{e} follows from the assumptions on $\eta$ and $m$.
Hence, by a comparison, there exists a coupling between $(U_{s(t)})$ and another process $(V_s)$ such that $U_{t} \leq V_{s(t)}$ where $V_0 = 0$ and
\begin{align*}
\d{V_s}
&= \left[3 (\Dmin + \eta) (1 + G_{t(s)}) + \frac{\epsilon \sqrt{\eta}}{2} + \frac{V_s\left(1 - \epsilon\right)}{2}\right] + \frac{1}{2}\sqrt{\eta \gamma_s^2} \d{W}_s \,,
\end{align*}
where $1 \leq \gamma_s^2 \leq (1 + 4 G_{t(s)})(1 + V_s^2)$.

\paragraph{Exact solution}
Recall now that
\begin{align}
G_t = \frac{1}{m} \exp\left(\frac{S_t + Z_t}{2} + \frac{S_t}{m}\right) 
\leq \frac{8}{m} \exp\left(s(t)/2\right)\,,
\label{eq:G1}
\end{align}
where we used \cref{def:tau} to argue that for $t \leq \tau$, $S_t \leq s(t) + 1$ and the assumption that $(s_{\max} + 1) / m \leq 1$.
Alternatively, $G_t$ can be bounded in another way for $t \leq \tau$ by
\begin{align}
G_t = \frac{1}{m} \exp\left(\frac{S_t + Z_t}{2} + \frac{S_t}{m}\right) 
&\leq \frac{200}{m \sqrt{\eta}} \,,
\label{eq:G2}
\end{align}
which again follows from \cref{def:tau} and \cref{eq:Z} and the assumption that $(s_{\max}+1)/m\leq 1$.
By the method of integrating factors (\cref{rem:factors}),
\begin{align}
V_s = \exps^{(1-\epsilon) s/2} \left[\frac{\sqrt{\eta}}{96} \int_0^s \exps^{-(1-\epsilon)u/2} (1 + G_{t(u)}) \d{u} + \sqrt{\eta} W_{[M]_s}\right]\,,
\label{eq:V}
\end{align}
where $[M]_s$ is the quadratic variation of the martingale
\begin{align*}
M_s = \frac{1}{2} \int_0^s \exps^{-(1-\epsilon)u/2} \gamma_u \d{W}_u\,.
\end{align*}
Let $\kappa = \inf\{s : V_s = 1 \text{ or } [M]_s = 3/4\}$.

\paragraph{Constant probability event}
By definition, $[M]_s \leq 3/2$ for $s \leq \kappa$.
Consider the event $E$ that $W_u + u \in [-1/8, 1/8]$ for all $u \leq 3/2$.
An elementary argument shows that $\bbP(E) > c$.
For the remainder we show the following hold on $E$ for all $s \leq s_{\max}$:
\begin{enumerate}
\item $[M]_s < 3/4$.
\item $V_s < 1$.
\item $V_s \leq -\frac{\sqrt{\eta}}{40} \exps^{(1-\epsilon)s/2}$.
\end{enumerate}
These imply that $\kappa = s_{\max}$ and in particular by our comparison to $(U_{s(t)})$ the claim holds.

\paragraph{Bounding the quadratic variation}
Note that by definition, if $s \leq \kappa$, then $V_s < 1$ and by \cref{eq:V},
\begin{align}
V_s^2 \leq \max\left(1, 4\eta \exps^{(1 - \epsilon)s}\right)\,,
\label{eq:Vs}
\end{align}
where we used the fact that when $s \leq \kappa$, then $[M]_s \leq 3/2$ and hence $W_{[M]_s} \geq -2$.
By definition the quadratic variation is
\begin{align}
[M]_s 
&= \frac{1}{4} \int_0^s \exps^{-(1 - \epsilon)u} \gamma_u^2 \d{u} \,. 
\label{eq:Ms}
\end{align}
Since all the terms are positive and $\gamma_u \geq 1$, it immediately follows that $[M]_s \geq (1 - \exps^{-s})/4$. 
Upper bounding is more delicate. Since $\gamma_t \in [0,4]$,
\begin{align*}
[M]_s
&\exl{a}\leq \frac{1}{4} \int_0^s \exps^{-(1-\epsilon)u} (1 + 4 G_{t(u)})(1+V_u^2) \d{u} \\
&\exl{b}\leq \frac{1}{4}\int_0^s \exps^{-(1-\epsilon)u} (1 + 4 G_{t(u)})(2 + 4\eta \exps^{(1-\epsilon) u}) \d{u} \\
&\exl{c}\leq \frac{1}{4} \int_0^s \exps^{-(1-\epsilon)u} \left(1 + \frac{800}{m} \min\left(\frac{1}{\sqrt{\eta}}, \exps^{u/2}\right)\right) \left(2 + 4\eta \exps^{(1-\epsilon) u}\right) \d{u} \\
&\exl{d}\leq \frac{1- \exps^{-(1-\epsilon)s}}{2(1-\epsilon)} + 1/8 \\
&\exl{e}\leq 3/4 \,.
\end{align*}
where \ex{a} follows from \cref{eq:Ms},
\ex{b} from \cref{eq:Vs},
\ex{c} from \cref{eq:G1} and \cref{eq:G2},
\ex{d} by computing the integral and naive bounding and \ex{e} is immediate.
Therefore $(1 - \exps^{-s})/2 \leq [M]_s \leq 3/2$ for all $s \leq s_{\max}$.

\paragraph{Bounding the SDE}
Our exact solution is
\begin{align*}
V_s 
&= \exps^{(1-\epsilon) s/2} \left[\frac{\sqrt{\eta}}{96}\int_0^s \exps^{-(1-\epsilon)u/2} (1 + G_{t(u)}) \d{u} + \sqrt{\eta} W_{[M]_s}\right] \\
&\exl{a}\leq \exps^{(1-\epsilon) s/2} \left[\frac{\sqrt{\eta}}{96} \int_0^s \exps^{-(1-\epsilon)u/2} \left(1 + \frac{8}{m} \exps^{u/2}\right) \d{u} + \sqrt{\eta} W_{[M]_s}\right] \\
&\exl{b}\leq \exps^{(1-\epsilon) s/2} \left[\frac{\sqrt{\eta}}{96} \int_0^s \exps^{-u/2 + \epsilon s_{\max}/2} \left(1 + \frac{8}{m} \exps^{u/2}\right) \d{u} + \sqrt{\eta} W_{[M]_s}\right] \\
&\exl{c}\leq \exps^{(1-\epsilon) s/2} \left[\frac{\sqrt{\eta}}{48} \int_0^s \exps^{-u/2} \left(1 + \frac{8}{m} \exps^{u/2}\right) \d{u} + \sqrt{\eta} W_{[M]_s}\right] \\
&\exl{d}\leq \exps^{(1-\epsilon) s/2} \left[\frac{\sqrt{\eta}}{48} \left(2 + \frac{8 s_{\max}}{m}\right) + \sqrt{\eta} W_{[M]_s}\right] \\
&\exl{e}\leq \exps^{(1-\epsilon) s/2} \left[\frac{\sqrt{\eta}}{16} + \sqrt{\eta} W_{[M]_s}\right] \\
&\exl{f}\leq \exps^{(1-\epsilon) s/2} \left[\frac{\sqrt{\eta}}{16} + \sqrt{\eta} (\exps^{-s}/4 - 1/8)\right] \\
&\exl{g}< \exps^{(1-\epsilon) s/2} \sqrt{\eta} (\exps^{-s}/4 - 1/16)\,. 
\end{align*}
where \ex{a} follows from \cref{eq:G2}.
\ex{b} since $s \leq s_{\max}$.
\ex{c} since $\exp(\epsilon s_{\max}/2) \leq 2$.
\ex{d} by crudely bounding the integral.
\ex{e} since $m \geq 8 s_{\max}$.
\ex{f} since $W_u \leq 1/8 - u$ and $[M]_s \geq (1 - \exps^{-s})/4$. 
\ex{g} since $\eta < c$ and $\sqrt{\eta} \geq C \Dmin$.
Hence, with constant probability, whenever $t \leq \tau$, then $S_{t(s)} \in ((1 - \Dmin)s - 1, s + 1)$ and
\begin{align*}
Z_t = 2 \arcsinh(U_t) \leq 2 \arcsinh(V_{s(t)}) < 2\arcsinh\left(\exps^{(1-\epsilon)s/2} \sqrt{\eta} (\exps^{-s}/4 - 1/16)\right)\,,
\end{align*}
which completes the proof.

\newpage
\section{Experiments}\label{app:exp}

We collect here the experiments in the lower bound construction in \cref{thm:lower} but using discrete time policy gradient as outlined in \cref{alg:pg}.

\begin{figure}[h!]
\centering
\includegraphics[width=13cm]{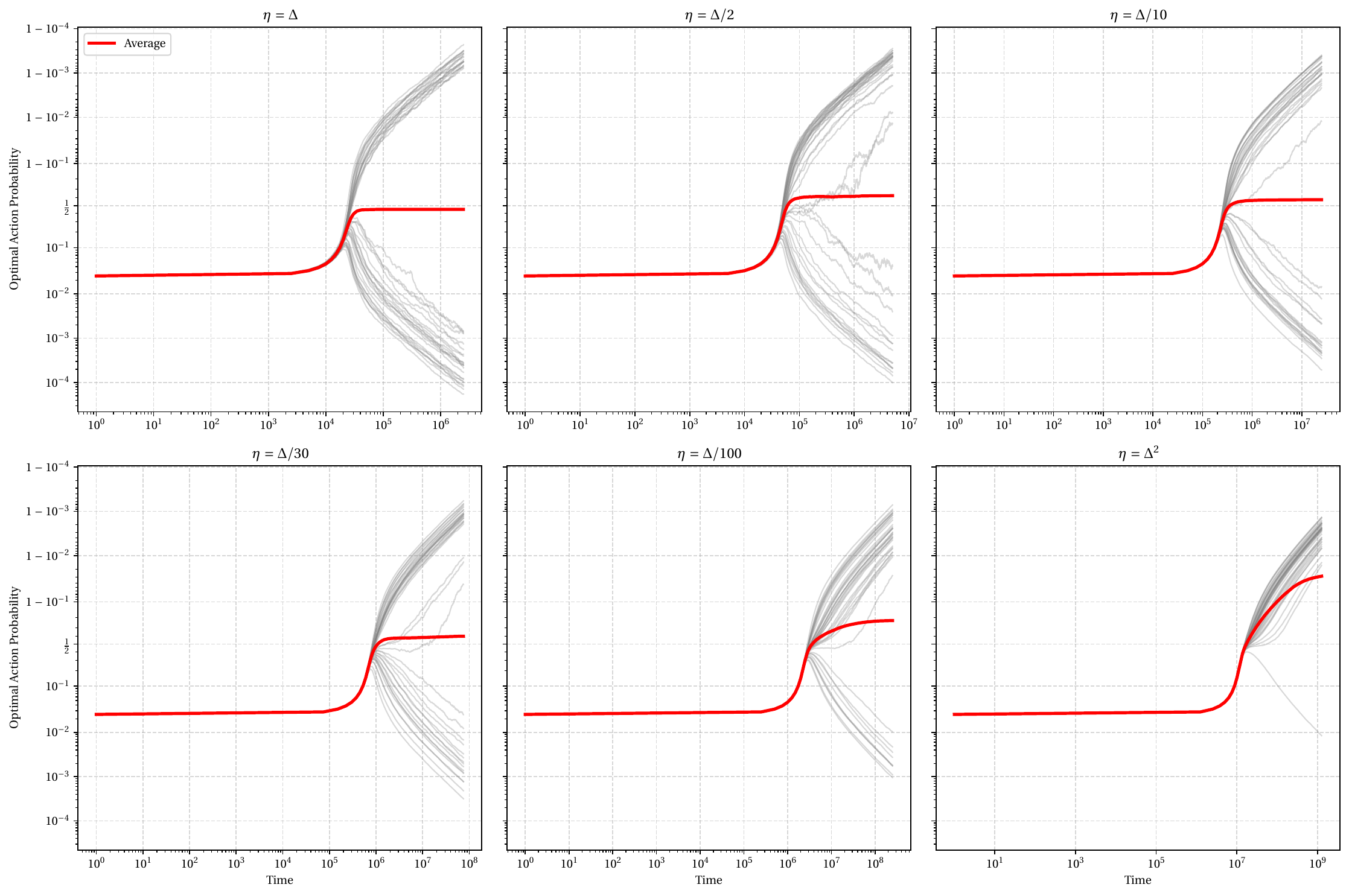}
\caption{
The plot shows 40 trajectories of $\pi_{t,1}$ produced by \cref{alg:pg} on the instance of \cref{thm:lower} for 6 different learning rates with $\Dmin = 0.002$.
}\label{fig:lower}
\end{figure}

\begin{figure}[h!]
\centering
\includegraphics[width=13cm]{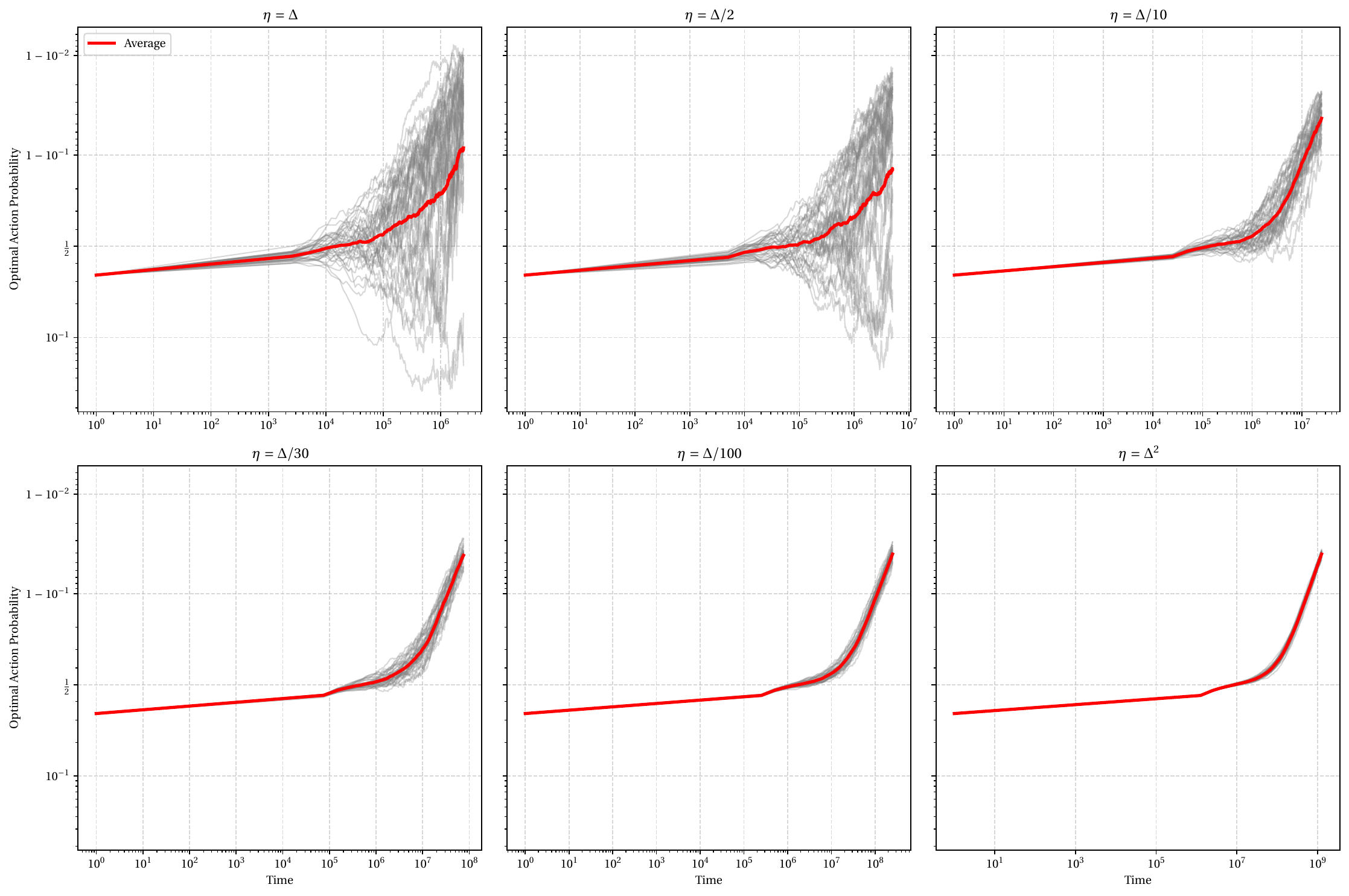}
\caption{
The figure shows the results for the same experiment as in \cref{fig:lower} but with $k = 3$, suggesting that the logarithmic number of arms used in the lower bound is needed.
}\label{fig:lower3}
\end{figure}

\end{document}